\pdfoutput=1
\documentclass[letterpaper]{article} 
\usepackage{aaai24}  
\usepackage{times}  
\usepackage{helvet}  
\usepackage{courier}  
\usepackage[hyphens]{url}  
\usepackage{graphicx} 
\usepackage{natbib}  
\usepackage{caption} 
\usepackage{algorithm}
\usepackage{algorithmic}

\usepackage{nicefrac}
\usepackage{amsmath}
\usepackage{amsthm}
\usepackage{amsfonts}
\usepackage{multirow}
\usepackage{booktabs}
\usepackage{subcaption}
\usepackage{colortbl}
\usepackage[dvipsnames]{xcolor}
\newtheorem{proposition}{Proposition}
\definecolor{tabhighlight}{HTML}{e5e5e5}

\newcommand{\tableCellHeight}{1}
\newcommand{\tabstyle}[1]{
	\setlength{\tabcolsep}{#1}
	\renewcommand{\arraystretch}{\tableCellHeight}
	\centering
	\small
}

\usepackage{newfloat}
\usepackage{listings}
\DeclareCaptionStyle{ruled}{labelfont=normalfont,labelsep=colon,strut=off} 
\floatstyle{ruled}
\newfloat{listing}{tb}{lst}{}
\floatname{listing}{Listing}
\title{Weak Distribution Detectors Lead to Stronger Generalizability of Vision-Language Prompt Tuning}

\author {
    Kun Ding\textsuperscript{\rm 1,\rm 3},
    Haojian Zhang\textsuperscript{\rm 2},
    Qiang Yu\textsuperscript{\rm 3},
    Ying Wang\textsuperscript{\rm 1,\rm 3}\footnote{Corresponding author.},
    Shiming Xiang\textsuperscript{\rm 1,\rm3},
    Chunhong Pan\textsuperscript{\rm 3}
}
\affiliations {
    \textsuperscript{\rm 1}State Key Laboratory of Multimodal Artificial Intelligence Systems, Institute of Automation, Chinese Academy of Sciences\\
    \textsuperscript{\rm 2}Engineering Laboratory for Intelligent Industrial Vision, Institute of Automation, Chinese Academy of Sciences\\
    \textsuperscript{\rm 3}Research Center of Aerospace Information, Institute of Automation, Chinese Academy of Sciences\\
    \{kun.ding, zhanghaojian2014, qiang.yu\}@ia.ac.cn, \{ywang, smxiang, chpan\}@nlpr.ia.ac.cn
}

\usepackage{bibentry}
\nocopyright

\begin{document}

\maketitle

\begin{abstract}
We propose a generalized method for boosting the generalization ability of pre-trained vision-language models (VLMs) while fine-tuning on downstream few-shot tasks. The idea is realized by exploiting out-of-distribution (OOD) detection to predict whether a sample belongs to a base distribution or a novel distribution and then using the score generated by a dedicated competition based scoring function to fuse the zero-shot and few-shot classifier. The fused classifier is dynamic, which will bias towards the zero-shot classifier if a sample is more likely from the distribution pre-trained on, leading to improved base-to-novel generalization ability. Our method is performed only in test stage, which is applicable to boost existing methods without time-consuming re-training. Extensive experiments show that even weak distribution detectors can still improve VLMs' generalization ability. Specifically, with the help of OOD detectors, the harmonic mean of CoOp~\cite{CoOp} and ProGrad~\cite{ProGrad} increase by 2.6 and 1.5 percentage points over 11 recognition datasets in the base-to-novel setting.
\end{abstract}

\section{Introduction}
Many large vision-language models (VLMs) trained on web-scale image-text pairs have emerged continuously in recent years (e.g., CLIP~\cite{CLIP}, ALIGN~\cite{ALIGN}) and they are proved to be very useful on many pure vision and multi-modal tasks, such as few-shot learning~\cite{CoOp}, image segmentation~\cite{DenseCLIP} and image captioning~\cite{ClipCap}. The huge parameter count of these models poses extreme challenges for classical fine-tuning techniques, which are both parameter inefficient and of low generalization capability. Vision-language prompt tuning (VLPT) has recently been proposed as an alternative technique~\cite{CoOp,zhou2022cocoop,ProGrad,kgcoop23}, which is demonstrated to be more parameter efficient. 

\begin{figure}[!t]
	\centering
	\includegraphics[width=0.85\columnwidth]{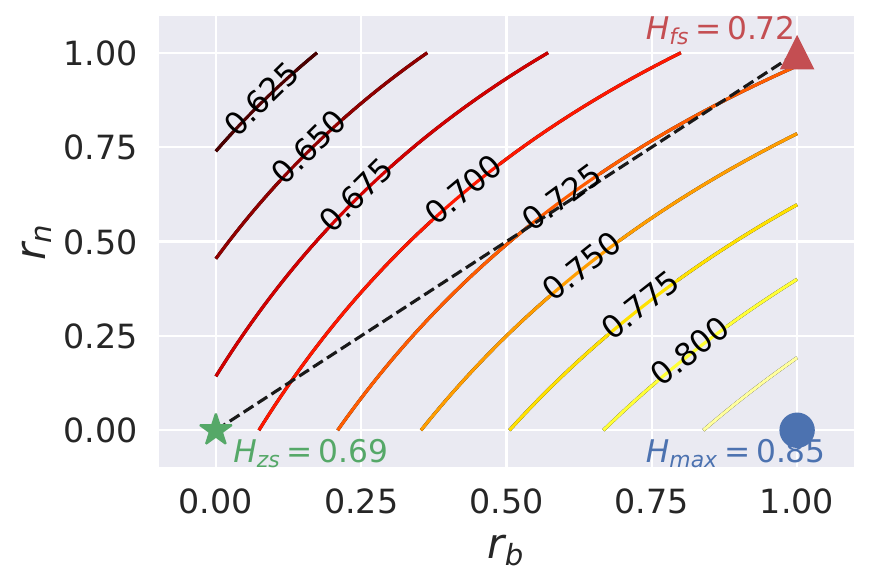}
	\caption{The relation between harmonic mean and ID prediction accuracies when $\alpha=\infty$ (in Eq.~\ref{eq:score_func}). $r_b$ and $1-r_n$ are the in-distribution prediction accuracy of samples in base and novel set. The maximum harmonic mean $H_{max}$ is acquired when $r_b=1$ and $r_n=0$. $H_{fs}$ and $H_{zs}$ denote the harmonic mean obtained by applying the few-shot and zero-shot classifier on all test samples, respectively.}
	\label{fig:harm_mean}
\end{figure}

However, it is hard for VLPT to balance the performance between base and novel classes. The first VLPT method CoOp~\cite{CoOp} achieves superior performance on base classes of 11 recognition datasets. Compared to zero-shot CLIP, the mean accuracy on novel classes drops 11 points (from 74.22\% to 63.22\%). CoCoOp~\cite{zhou2022cocoop} alleviates the above problem by conditioning prompts on image feature. CoCoOp improves the novel class accuracy in the sacrifice of base class accuracy and the model design is quite heavy weighted. ProGrad~\cite{ProGrad} prevents prompt tuning from forgetting the general knowledge in VLMs by only updating the prompt whose gradient is aligned to the general direction. KgCoOp~\cite{kgcoop23} introduces an auxiliary loss to reduce the discrepancy between the textual embeddings generated by learned prompts and hand-crafted prompts. 

Different from the above methods, we propose a new viewpoint for addressing the trade-off problem between base and novel classes. This viewpoint is motivated by the fact: zero-shot CLIP excels on novel classes while CoOp excels on base classes~\cite{zhou2022cocoop}. Therefore, if it is possible to know a sample comes from the base distribution or the novel distribution beforehand and choose the right classifier accordingly, such a trade-off problem will be addressed perfectly. The relation between the base-to-novel evaluation metric (harmonic mean) and the in-distribution (ID) prediction accuracies is studied in Fig.~\ref{fig:harm_mean}, which implies that the harmonic mean of the fused classifier with binary scoring function (in Eq.~\ref{eq:score_func}) increases with ID detection accuracies ($r_b$ and $1-r_n$ on base and novel set respectively). The problem of discriminating a sample from which distribution is similar to the out-of-distribution (OOD) detection problem. As such, we could convert the base-to-novel trade-off problem to an OOD detection problem.

Specifically, we design a competition based scoring function which exploits the ID scores generated by existing OOD detection methods to generate a score between 0 and 1 for a test sample and then use this score to dynamically fuse the pre-trained zero-shot classifier and the few-shot classifier tuned on limited training data. Different from the scoring functions in prior OOD detection methods, the proposed scoring function is based on the comparison between the outputs of two classifiers. The proposed fusion method is dynamic, i.e. the fusion weights vary from different samples, which is better than static method. Also, the fusion is implemented in test stage, making it easy for the integration in existing methods without re-training.

We conduct extensive experiments on 15 recognition datasets under different settings. In the base-to-novel setting, our method improves the harmonic mean accuracy of CoOp and ProGrad by 2.6\% and 1.5\%, respectively. In the domain generalization setting, the mean accuracies over target domains are competitive, either. The main contributions of this work are listed as follows:
\begin{itemize}
	\item We propose a new viewpoint for dealing with the base-to-novel generalization problem of VLPT by introducing OOD detection.
	\item We propose a competition based scoring function for fusing the few-shot and zero-shot classifier dynamically.
	\item We demonstrate by experiment even weak distribution detectors can still improve the generalizability of
	existing VLPT methods.
\end{itemize}

\section{Related Work}

\textbf{Vision-Language Models:} Recent years have witnessed lots of vision-language models, such as CLIP~\cite{CLIP}, ALIGN~\cite{ALIGN}, BLIP~\cite{BLIP}. Benefited from massive image-text pairs, large vision-language models have demonstrated huge potential on many vision tasks, including but not limited to zero-shot learning~\cite{CLIP, TTPT}, few-shot learning~\cite{CoOp}, image segmentation~\cite{DenseCLIP}, image generation~\cite{DALLE2}, etc. The key problem of exploiting so many pre-trained VLMs is how to fine-tune them on downstream tasks in a parameter efficient way.

\noindent\textbf{Vision-Language Prompt Tuning:} As one of the most popular parameter efficient tuning techniques, prompt tuning is first proposed in NLP and subsequently introduced into vision-language domain by~\citet{CoOp}. Instead of manually designing prompts, prompt tuning learns soft context embeddings from data, which avoids cumbersome handcrafts. Based on CoOp, lots of efforts have been made to solve the problems in certain aspects or expand application scope. For example, CoCoOp~\cite{zhou2022cocoop}, ProGrad~\cite{ProGrad} and KgCoOp~\cite{kgcoop23} are devised to reduce the loss of generalization on unseen tasks. DualCoOp~\cite{DualCoOp} extends CoOp to the multi-label recognition task. SoftCPT~\cite{SoftCPT} exploits the cross-task knowledge to further enhance the generalization ability. PLOT~\cite{PLOT} introduces optimal transport to achieve fine-grained matching between multiple prompts and visual regions. Prompt tuning is also introduced in image segmentation~\cite{DenseCLIP} and video relation detection~\cite{CompPT}.

\noindent\textbf{OOD Detection:} OOD detection is to identify outliers in an open sample space not belonging to any seen training classes, which is especially vital for safety-critical circumstances. According to~\citet{OODsurvey}, OOD detection methods can be categorized into classification-based, density-based, distance-based and reconstruction-based methods. Post-hoc detection is an important member of classification-based methods that does not need extra training data. Given a trained model, post-hoc methods calculate a score to indicate the ID-ness or OOD-ness. \citet{MSP} proposed the maximum softmax probabilities (MSP), which is a widely-used baseline. Considering that MSP is problematic for large-scale settings with hundreds or thousands of classes, MaxLogit~\cite{MaxLogit} that uses the maximum unnormalized logit as anomaly score was proposed. \citet{EnergyBasedOOD} proposed an energy based indicator that has good theoretical interpretability. More recently, \citet{DecoupleMaxLogit} decoupled MaxLogit into cosine similarity and logit norm and introduced a tunable factor between them for better performance. Based on these works we propose a novel scoring function tailored for VLPT, which considers two classifiers jointly, i.e. zero-shot and few-shot classifier. This is different from the above methods that only use one trained classifier to design scoring function.

\noindent\textbf{Dynamic Neural Networks:} Designing instance-aware dynamic neural networks is an important research topic in the literature. Dynamic convolution~\cite{ChenDLCYL20,HouK21,li2022omnidimensional} learns a linear combination of multiple static kernels with input-dependent dynamic attention. Such a design effectively increases model capability without increasing FLOPs a lot. Dynamic quantization~\cite{DyQuan} learns quantization configuration for each image and layer individually, which further reduces computation cost. The Mixture-of-Expert (MoE) layer contains many expert sub-networks and learns a gating function to route individual inputs to some of the experts~\cite{ShazeerMMDLHD17}. It substantially increases the model capacity without introducing large computation overhead. Our proposed method dynamically fuses two classifiers, it can also be seen as a dynamic network. However, it is a test-time method as it does not need re-training.

\section{Background}
\citet{CLIP} demonstrated the contrastive learning based vision-language model CLIP has strong zero-shot recognition ability. The good generability is believed to be benefited from natural language supervision and diverse large-scale training data. CLIP has a two-stream structure: one image encoder for representing image as visual embedding and one text encoder for representing text as text embedding. Due to the contrastive learning based design, CLIP can be easily adapted for zero-shot and few-shot recognition.

\noindent\textbf{Zero-shot CLIP.} We denote the image and text encoder in CLIP as $E_{img}$ and $E_{txt}$, respectively. For a zero-shot classification task of $N$ classes, \citet{CLIP} used the classifier generated by the text encoder for classification. The $N$ class names are assumed to be accessible and the classifier weight (an unnormalized column vector) associated to the $i$-th class ($i\in \{1,2,\cdots,N\}$) is obtained by encoding its class name, i.e., $\mathbf{w}_{i}^{zs}=E_{txt}(\text{``[class name]''})$. Adding a prefix string to class name is demonstrated to be useful, that is, $\mathbf{w}_{i}^{zs}=E_{txt}(\text{``a photo of a [class name]''})$. Collecting the weights of all classes into a matrix results in $\mathbf{W}^{zs}=[\mathbf{w}_1^{zs},\cdots,\mathbf{w}_{N}^{zs}]^T$. Given an image $I$, its visual embedding $\mathbf{x}$ is extracted by $\mathbf{x}=E_{img}(I)$, which is an unnormalized vector with a same length to $\mathbf{w}_{i}^{zs}$. To classify $I$, zero-shot CLIP computes the posterior probability of a certain class $y$ as follows:
\begin{align}
p^{zs}(y|\mathbf{x}) = {\text{exp}(o^{zs}_y(\mathbf{x}))} / {\sum\nolimits_{i=1}^N \text{exp}(o^{zs}_i(\mathbf{x}))},\label{eq:zs_classify}
\end{align}
where $o^{zs}_y(\mathbf{x})=d(\mathbf{x},\mathbf{w}_y^{zs})/\tau$ is the logit of class $y$ with $\tau$ a temperature parameter and $d(\cdot)$ the cosine similarity.

\noindent\textbf{Context Optimization (CoOp).} Vision-language prompt tuning (VLPT) methods like CoOp and its successors (ProGrad, KgCoOp, etc.) extend zero-shot CLIP for few-shot classification. Meantime, the fixed prompt prefix is upgraded as a learnable prompt context. They defines a sequence of $M$ vectors as the context denoted as a matrix $\mathbf{V}=[\mathbf{v}_1,\cdots,\mathbf{v}_{M}]$. For the $i$-th class, the class token embeddings can be represented as a matrix $\mathbf{C}_i$. The classifier weight associated to class $i$ is computed by encoding $[\mathbf{V}, \mathbf{C}_i]$ as a text embedding: $\mathbf{w}^{fs}_i=E_{txt}([\mathbf{V}, \mathbf{C}_i])$. Collecting the weights of all classes into a matrix results in $\mathbf{W}^{fs}=[\mathbf{w}_1^{fs},\cdots,\mathbf{w}_{N}^{fs}]^T$. To classify an image of encoded feature $\mathbf{x}$, the following posterior probability of a certain class $y$ is calculated:
\begin{align}
p^{fs}(y|\mathbf{x}) = {\text{exp}(o^{fs}_y(\mathbf{x}))} / {\sum\nolimits_{i=1}^N \text{exp}(o^{fs}_i(\mathbf{x}))},\label{eq:fs_classify}
\end{align}
where $o^{fs}_y(\mathbf{x}) = d(\mathbf{x},\mathbf{w}_y^{fs})/\tau$ is the logit of the $y$-th class. VLPT methods optimize the context $\mathbf{V}$ while fixing the image and text encoder. To align with upstream task, the negative log likelihood loss is widely adopted.

\section{Method}
\label{sec:method}
The flowchart of the proposed method is presented in Fig.~\ref{fig:method}. For an image, we extract the visual embedding vector. With hand-crafted prompts, we compute the zero-shot classifier. With learnable prompts, we compute the few-shot classifier. Subsequently, two posterior probability distributions, which serve as the inputs of OOD detector, are computed using these two classifiers. The OOD detector outputs in-distribution score for a test sample. The two in-distribution scores, $ids^{zs}(\mathbf{x})$ and $ids^{fs}(\mathbf{x})$, are compared by a scoring function $s(\mathbf{x})$. This final score is used to dynamically weight the two classifiers. In the following, we detail the key ingredients of the proposed method. Besides, we further analyze the relation between harmonic mean accuracy and OOD detection accuracies.

\begin{figure}[t]
	\centering
	\includegraphics[width=1\columnwidth]{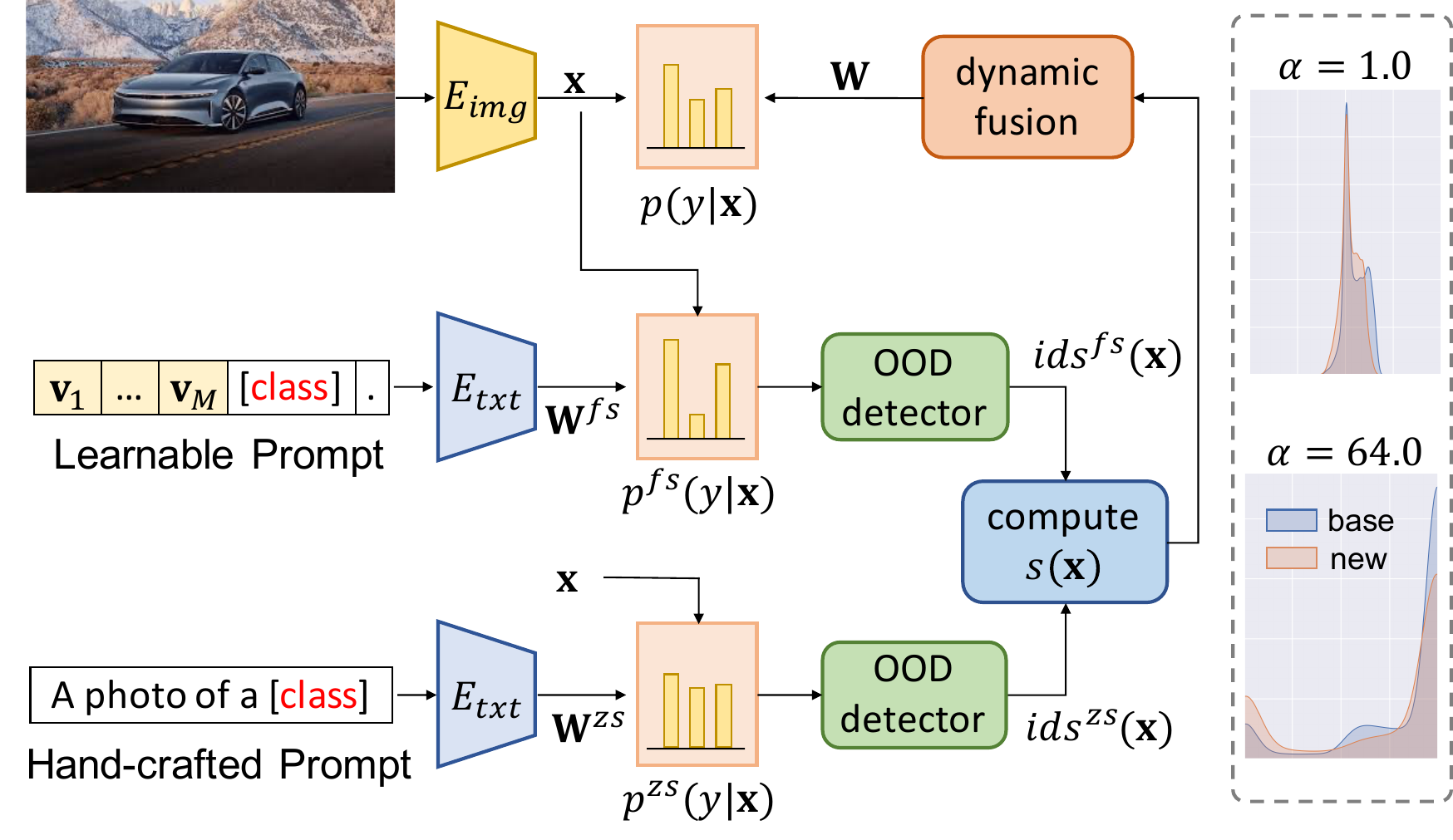}
	\caption{The framework of the proposed method.}
	\label{fig:method}
\end{figure}

\subsection{OOD Detectors}
Existing OOD detection methods calculate a score to reflect ID-ness or OOD-ness. By choosing a valid threshold and thresholding the score with it, binary prediction can be obtained. This work adopts post-hoc methods due to their simplicity. The original scores proposed in these method are converted to values reflecting ID-ness. We use three existing methods (MSP, MaxLogit and Energy) and also consider a new entropy based method. Actually, many other OOD detection methods can also be used.

MSP~\cite{MSP} uses the maximum softmax probability of a classifier as the ID score, that is
\begin{align}
ids^{*}(\mathbf{x}) = \max_i p^{*}(i|\mathbf{x}),
\end{align}
where $*$ is $zs$ or $fs$. $ids^{zs}(\mathbf{x})$ and $ids^{fs}(\mathbf{x})$ denote the in-distribution score corresponding to the zero-shot and few-shot classifier, respectively.

MaxLogit~\cite{MaxLogit} uses the maximum logit of a classifier as the ID score, that is
\begin{align}
ids^{*}(\mathbf{x}) = \tau \cdot \max_i o^{*}_i(\mathbf{x}),
\end{align}
where $\tau$ is multiplied for normalizing the final score into the range $[-1,1]$.

In~\citet{EnergyBasedOOD}, the free energy function over $\mathbf{x}$ in terms of the denominator of the softmax function is used as the OOD score. Accordingly, the negative free energy is used as the ID score:
\begin{align}
ids^{*}(\mathbf{x}) = (\log N+1/\tau)^{-1} \cdot\!\tau\!\cdot\!\log \sum\nolimits_i^N \exp(o^{*}_i(\mathbf{x})), 
\end{align}
where $(\log N+1/\tau)^{-1}$ is a normalizing factor.

The entropy of a softmax probability distribution reflects the OOD-ness of a sample. Larger entropy means the sample is more likely from a different distribution. We thus use the negative entropy as the ID score:
\begin{align}
ids^{*}(\mathbf{x}) = (\log N)^{-1} \cdot \sum\nolimits_i^{N} p^{*}(i|\mathbf{x})\log p^{*}(i|\mathbf{x}),
\end{align}
where $(\log N)^{-1}$ is a normalizing factor.

\subsection{Competition Based Scoring Function}
We propose a scoring function to indicate a sample is more likely from the base or novel distribution. The intuition is that if $ids^{fs}(\mathbf{x})-ids^{zs}(\mathbf{x})$ is larger, it has a higher probability that the sample is from the base distribution (also the distribution fine-tuned on). To obtain a valid weight, the score should be normalized into [0,1], thus, we introduce a sigmoid function $\sigma(x)$ to normalize the above difference. Formally, the scoring function is defined as
\begin{align}
s(\mathbf{x})=\sigma(\alpha\cdot (ids^{fs}(\mathbf{x})-ids^{zs}(\mathbf{x}))),\label{eq:score_func}
\end{align}
where $\alpha$ is a tunable scaling factor, which controls the dynamic range of the score (ref. Fig.~\ref{fig:method}). Not that, applying sigmoid function over the difference equals to applying the softmax function over $ids^{fs}(\mathbf{x})$ and $ids^{zs}(\mathbf{x})$, implying the existence of competition. Note that, with a sufficiently large $\alpha$, the sigmoid function equals to a Heaviside step function. In this case, it actually makes a binary prediction.

\subsection{Dynamic Fusion}
When a sample is more likely from the novel distribution, it is better to use the zero-shot classifier; conversely, when it is more likely from the base distribution, it is better to use the few-shot classifier. A more general method is weighted combining the few-shot and zero-shot classifier with $s(\mathbf{x})$:
\begin{align}
\mathbf{W}=s(\mathbf{x})\cdot \mathbf{W}^{fs} + (1-s(\mathbf{x}))\cdot \mathbf{W}^{zs}. \label{eq:fuse_weight}
\end{align}
By this way, the fused classifier is instance-conditioned, which dynamically changes with different inputs. Note that other strategies such as weighted combination of the posterior probabilities from the two classifiers are also possible.

\subsection{Analysis of Harmonic Mean}
The harmonic mean performance of the classifier in Eq.~\ref{eq:fuse_weight} has direct relation to OOD detection performance. In the base-to-novel evaluation setting of VLPT, there are two test sets of non-overlapping classes: base set and novel set. Let us denote $r_b$ and $r_n$ as the probability of predicting a sample in the base and novel set as a sample from base distribution, respectively. The classification accuracy of the few-shot classifier $\mathbf{W}^{fs}$ and the zero-shot classifier $\mathbf{W}^{zs}$ on base set are denoted as $p_1$ and $p_0$, respectively. The classification accuracy of the few-shot classifier $\mathbf{W}^{fs}$ and the zero-shot classifier $\mathbf{W}^{zs}$ on novel set are denoted as $q_1$ and $q_0$, respectively. We have the following proposition:

\begin{proposition}
	The harmonic mean accuracy $H$ of the fused classifier $\mathbf{W}$ with the Heaviside step function used on base and novel set can be represented as:
	\begin{align}
	H &= {2P_bP_n}/{(P_b+P_n)},\\
	P_b&=p_1r_b+p_0(1-r_b), \\
	P_n&=q_1r_n+q_0(1-r_n),
	\end{align}
\end{proposition}

\begin{proof}
	The proposition can be proved by the law of total probability.
\end{proof}

\begin{table*}[!t]
	\tabstyle{4pt}
	\begin{subtable}[t]{.3\textwidth}
		\centering
		\renewcommand{\arraystretch}{0.9}
		\begin{tabular}{l cc|c}
			\toprule
			& Base & Novel & H \\
			\midrule
			CLIP & 69.16 & 74.13 & 71.56 \\
			CoOp & {82.29} & 67.63 & 74.24 \\
			CoOp+ & 82.20 & 72.15 & 76.85 \\
			CoCoOp & 80.12 & 72.36 & 76.04 \\
			CoCoOp+ & 79.26 & \textbf{74.37} & 76.74 \\
			ProGrad & 82.27 & 70.10 & 75.70\\
			ProGrad+ & 80.74 & 73.95 & {77.20}\\
			KgCoOp & 81.57 & 72.87 & 76.97 \\
			KgCoOp+ & 79.81 & 74.26 & 76.94\\
			Pro+Kg & \textbf{82.43} & 73.31 & \textbf{77.60}\\
			\bottomrule
		\end{tabular}
	\caption{{Average over 11 datasets}.}
	\end{subtable}
	\vspace{0.1em}
	\begin{subtable}[t]{.3\textwidth}
		\centering
		\renewcommand{\arraystretch}{0.9}
		\begin{tabular}{l cc|c}
			\toprule
			& Base & Novel & H \\
			\midrule
			CLIP & 72.38 & 68.10 & 70.17\\
			CoOp & {76.46} & 65.47 & 70.54\\
			CoOp+ & 76.21 & 68.87 & 72.35 \\
			CoCoOp & 75.86 & {70.59} & {73.13} \\
			CoCoOp+ & 75.16 & \textbf{70.65} & 72.84 \\
			ProGrad & {76.73}& 67.64& 71.90\\
			ProGrad+ & 75.79 & 69.37 &72.44 \\
			KgCoOp & 75.80 & 69.81 & 72.68\\
			KgCoOp+ & 74.67 & 69.57& 72.03\\
			Pro+Kg & \textbf{76.79} & 69.84 & \textbf{73.15} \\
			\bottomrule
		\end{tabular}
	\caption{ImageNet.}
	\end{subtable}
	~
	\begin{subtable}[t]{.3\textwidth}
		\centering
		\renewcommand{\arraystretch}{0.9}
		\begin{tabular}{l cc|c}
			\toprule
			& Base & Novel & H \\
			\midrule
			CLIP & 94.58 & {94.10} & 94.34 \\
			CoOp & {95.74} & 93.65 & 94.68 \\
			CoOp+ & 96.10 & 94.10 &  95.09\\
			CoCoOp & 95.85 & 94.45 & {95.14} \\
			CoCoOp+ & 95.85 & 94.25 & 95.04\\
			ProGrad & \textbf{96.41} & 93.69& 95.03\\
			ProGrad+ & 95.98 & 94.60 &95.29 \\
			KgCoOp & 95.89 & {94.74} & {95.31}\\
			KgCoOp+ & 95.68 & 94.61& 95.14\\
			Pro+Kg & 96.18 & \textbf{94.81}& \textbf{95.49}\\
			\bottomrule
		\end{tabular}
	\caption{Caltech101.}
	\end{subtable}
	~
	\begin{subtable}[t]{.3\textwidth}
		\centering
		\renewcommand{\arraystretch}{0.9}
		\begin{tabular}{l cc|c}
			\toprule
			& Base & Novel & H \\
			\midrule
			CLIP & 91.05 & 97.16 & 94.01 \\
			CoOp & 94.99 & 93.06 & 94.02 \\
			CoOp+ & 94.92 & 96.36 & 95.63 \\
			CoCoOp & {95.00} & \textbf{97.81} & {96.38} \\
			CoCoOp+ & 94.99 & 97.50 & 96.23\\
			ProGrad & {95.44}& 95.15& 95.29\\
			ProGrad+ & 94.60 & 97.42 & 95.99\\
			KgCoOp & 95.14 & 97.62 & 96.36\\
			KgCoOp+ & 94.39 & 97.55 & 95.94\\
			Pro+Kg & \textbf{95.50} & 97.64 & \textbf{96.56}\\
			\bottomrule
		\end{tabular}
	\caption{OxfordPets.}
	\end{subtable}
	\vspace{0.1em}
	\begin{subtable}[t]{.3\textwidth}
		\centering
		\renewcommand{\arraystretch}{0.9}
		\begin{tabular}{l cc|c}
			\toprule
			& Base & Novel & H \\
			\midrule
			CLIP & 63.74 & {74.89} & 68.87 \\
			CoOp & \textbf{78.37} & 67.49 & 72.52 \\
			CoOp+ & 77.22 & 72.47 &{74.77} \\
			CoCoOp & 70.73 & 72.70 & {71.70} \\
			CoCoOp+ & 69.28 & 74.50 & 71.80\\
			ProGrad & 76.91& 70.55& 73.59\\
			ProGrad+ & 74.16 & 73.99 &74.07 \\
			KgCoOp & 74.21 & 74.77 &74.49 \\
			KgCoOp+ & 70.65 & \textbf{75.54} & 73.01\\
			Pro+Kg & 76.61 & 73.54 & \textbf{75.04}\\
			\bottomrule
		\end{tabular}
	\caption{StanfordCars.}
	\end{subtable}
	~
	\begin{subtable}[t]{.3\textwidth}
		\centering
		\renewcommand{\arraystretch}{0.9}
		\begin{tabular}{l cc|c}
			\toprule
			& Base & Novel & H \\
			\midrule
			CLIP & 70.24 & \textbf{77.98} & 73.91 \\
			CoOp & \textbf{97.23} & 62.40 & 76.02 \\
			CoOp+ & 96.46 & 73.05 & 83.14 \\
			CoCoOp & 94.79 & 68.44 & {79.49} \\
			CoCoOp+ & 92.81 & 75.01 & 82.97\\
			ProGrad & 96.43& 69.10& 80.51\\
			ProGrad+ & 93.10 & 76.02 & {83.70}\\
			KgCoOp & 95.87 & 72.75 &82.72 \\
			KgCoOp+ & 92.84 & 75.78 & 83.45\\
			Pro+Kg & 96.50 & 74.26 & \textbf{83.93}\\
			\bottomrule
		\end{tabular}
	\caption{Flowers102.}
	\end{subtable}
	~
	\begin{subtable}[t]{.3\textwidth}
		\centering
		\renewcommand{\arraystretch}{0.9}
		\begin{tabular}{l cc|c}
			\toprule
			& Base & Novel & H \\
			\midrule
			CLIP & 89.77 & 91.29 & 90.52 \\
			CoOp & 89.54 & 90.14 & 89.84 \\
			CoOp+ & 90.52 & 91.61 & 91.06 \\
			CoCoOp & {90.55} & {91.38} & {90.96} \\
			CoCoOp+ & {90.61} & \textbf{91.88} & \textbf{91.24}\\
			ProGrad & 90.46& 90.54& 90.50\\
			ProGrad+ & 90.55 & 91.46 & 91.00\\
			KgCoOp & 90.40 & 91.72 & 91.06\\
			KgCoOp+ & 90.23 & 91.62 & 90.92\\
			Pro+Kg & \textbf{90.70} & 91.74 & 91.22\\
			\bottomrule
		\end{tabular}
	\caption{Food101.}
	\end{subtable}
	\vspace{0.1em}
	\begin{subtable}[t]{.3\textwidth}
		\centering
		\renewcommand{\arraystretch}{0.9}
		\begin{tabular}{l cc|c}
			\toprule
			& Base & Novel & H \\
			\midrule
			CLIP & 27.44 & \textbf{35.77} & {31.06} \\
			CoOp & \textbf{39.44} & 27.20 & 32.20 \\
			CoOp+ & 39.13 & 30.81 & 34.48 \\
			CoCoOp & 36.08 & 33.05 & 34.50 \\
			CoCoOp+ & 34.73 & 35.01 & {34.87}\\
			ProGrad & 39.43& 27.90& 32.68 \\
			ProGrad+ & 38.64 & 31.49 & 34.70\\
			KgCoOp & 37.94 & 31.31 & 34.31\\
			KgCoOp+ & 36.49 & 34.94 & \textbf{35.70}\\
			Pro+Kg & 39.14 & 30.36 & 34.20\\
			\bottomrule
		\end{tabular}
	\caption{FGVCAircraft.}
	\end{subtable}
	~
	\begin{subtable}[t]{.3\textwidth}
		\centering
		\renewcommand{\arraystretch}{0.9}
		\begin{tabular}{l cc|c}
			\toprule
			& Base & Novel & H \\
			\midrule
			CLIP & 69.39 & 75.53 & 72.33 \\
			CoOp & {81.17} & 70.39 & 75.40 \\
			CoOp+ & 81.10 & 75.30 &  78.09\\
			CoCoOp & 79.52 & {76.62} & {78.04} \\
			CoCoOp+ & 78.21 & \textbf{77.97} & 78.09\\
			ProGrad & {81.36}& 73.95& 77.48\\
			ProGrad+ & 79.80 & 77.29 & {78.52}\\
			KgCoOp & 80.76 & 76.05 & 78.33\\
			KgCoOp+ & 78.60 & 77.29 & 77.94\\
			Pro+Kg & \textbf{81.84} & 76.98 & \textbf{79.34}\\
			\bottomrule
		\end{tabular}
	\caption{SUN397.}
	\end{subtable}
	~
	\begin{subtable}[t]{.3\textwidth}
		\centering
		\renewcommand{\arraystretch}{0.9}
		\begin{tabular}{l cc|c}
			\toprule
			& Base & Novel & H \\
			\midrule
			CLIP & 53.94 & \textbf{57.97} & 55.88 \\
			CoOp & {78.74} & 45.81 & 57.92 \\
			CoOp+ & 79.09 & 47.99 &  59.73\\
			CoCoOp & 74.65 & 53.66 & {62.44} \\
			CoCoOp+ & 73.77 & 55.60 & 63.41\\
			ProGrad & 77.31& 48.95& 59.94\\
			ProGrad+ & 75.15 & 55.88 & 64.10\\
			KgCoOp & \textbf{80.09} & 52.98 & 63.77\\
			KgCoOp+ & 78.20 & {56.88} & \textbf{65.86}\\
			Pro+Kg & 80.05 & 53.74 & 64.31\\
			\bottomrule
		\end{tabular}
	\caption{DTD.}
	\end{subtable}
	~
	\begin{subtable}[t]{.3\textwidth}
		\centering
		\renewcommand{\arraystretch}{0.9}
		\begin{tabular}{l cc|c}
			\toprule
			& Base & Novel & H \\
			\midrule
			CLIP & 57.24 & {64.08} & 60.47 \\
			CoOp & {89.21} & 63.08 & 73.90 \\
			CoOp+ & 89.02 & 66.25 &  75.97\\
			CoCoOp & 86.02 & 66.43 & {74.97} \\
			CoCoOp+ & 85.22 & \textbf{70.38} & 77.09 \\
			ProGrad & \textbf{90.76}& 65.25& 75.92\\
			ProGrad+ & 88.36 & 69.38 & \textbf{77.73}\\
			KgCoOp & 88.11 & 64.37 & 74.39 \\
			KgCoOp+ & 84.70 & 64.86 & 73.46 \\
			Pro+Kg & 89.92 & 67.82 & 77.32\\
			\bottomrule
		\end{tabular}
	\caption{EuroSAT.}
	\end{subtable}
	~
	\begin{subtable}[t]{.3\textwidth}
		\centering
		\renewcommand{\arraystretch}{0.9}
		\begin{tabular}{l cc|c}
			\toprule
			& Base & Novel & H \\
			\midrule
			CLIP & 71.04 & \textbf{78.53} & 74.60 \\
			CoOp & {84.35} & 65.21 & 73.56 \\
			CoOp+ & \textbf{84.47} & 76.80 &  \textbf{80.45}\\
			CoCoOp & 82.28 & 70.87 & {76.15} \\
			CoCoOp+ & 81.28 & 75.39 & 78.22\\
			ProGrad & 83.77& 68.36& 75.28\\
			ProGrad+ & 82.02 & 76.53 & 79.18\\
			KgCoOp & 83.09 & 75.39 & 79.05 \\
			KgCoOp+ & 81.46 & 78.20 & 79.80\\
			Pro+Kg & 83.56 & 75.68 & 79.43\\
			\bottomrule
		\end{tabular}
	\caption{UCF101.}
	\end{subtable}
\caption{Comparison of different methods in the base-to-novel generalization setting using 16 shots. H: Harmonic mean.}
\label{tab:base_to_novel}
\end{table*}

We plot the contour map of $H$ with respect to $r_b$ and $r_n$ in Fig.~\ref{fig:harm_mean} with $p_0=0.6, p_1=0.9, q_0=0.8, q_1=0.6$. Each point in the figure represents an OOD detector of certain performance. The black dashed line corresponds to random OOD detectors. When moving towards the lower right corner, $H$ increases gradually. A relatively weak classifier (e.g., the classfiers on the 0.750 contour) can still improve the harmonic mean accuracy of the zero-shot and few-shot classifier, i.e., $H_{zs}=0.69$ and $H_{fs}=0.72$, respectively. This implies the feasibility of the proposed method.

\begin{table}[t]
	\tabstyle{3.5pt}
		\renewcommand{\arraystretch}{0.9}
		\begin{tabular}{l cccccc}
			\toprule
			& Source & \multicolumn{5}{c}{Target} \\
			\cmidrule(lr){2-2} \cmidrule(lr){3-7}
			& IN & INV2 & IN-Sketch & IN-A & IN-R & Avg. \\
			\midrule
			CLIP & 66.71 & 60.89 & 46.09 & 47.77 & 73.98 & 57.18 \\
			\midrule
			CoOp & 71.71 & 64.07 & 46.96 & 48.84 & 74.54 &58.60 \\
			CoOp+  & 71.43 & 64.09 & 48.09 & 49.65 & 75.59 & 59.36 \\
			\midrule
			CoCoOp & 71.16 & 64.35 & 48.87 & 50.73 & 75.90 & 59.96\\
			CoCoOp+ & 70.44 & 63.79 & 48.72& 50.68& 76.03& 59.81\\
			\midrule
			ProGrad & \textbf{72.24} & 64.86 & 48.18 & 49.91 & 75.50 & 59.61 \\
			ProGrad+  & 71.12 & 64.05 & 48.21 & 49.89 & 75.78 & 59.48 \\
			\midrule
			KgCoOp & 70.67 & 63.71 & 48.72 & 50.39 & \textbf{76.68} & 59.88 \\
			KgCoOp+  & 69.60 & 62.91 & 48.04 & 49.62 & 76.04 & 59.15 \\
			\midrule
			Pro+Kg & 72.02 & \textbf{64.92} & \textbf{49.06} & \textbf{50.84} & 76.62 & \textbf{60.36} \\
			\bottomrule
		\end{tabular}
	\caption{Comparison of methods with and without OOD detection in domain generalization setting. IN=ImageNet.}
\label{tab:dg}
\end{table}

\section{Experiment}
Similar to prior work, we evaluate our method under two settings: 1) base-to-novel generalization within a dataset; 2) domain generalization across datasets with the same classes.

\subsection{Setup}
\textbf{Dataset.} Following prior works~\cite{CoOp}, for the base-to-novel experiment, we use 11 image classification datasets: ImageNet~\cite{ImageNet} and Caltech101~\cite{Caltech101} for generic object classification; Food101~\cite{Food101}, StanfordCars~\cite{StanfordCars}, OxfordPets~\cite{OxfordPets} and Flowers102~\cite{Flowers102} and FGVCAircraft~\cite{FGVCAircraft} for fine-grained visual classification; UCF101~\cite{UCF101} for action recognition; EuroSAT~\cite{EuroSAT} for satellite image classification; DTD~\cite{DTD} for texture classification; SUN397~\cite{SUN397} for scene recognition. For the domain generalization experiment, we use ImageNet as the source domain and other four datasets, i.e. ImageNetV2~\cite{ImageNetV2}, ImageNet-Sketch~\cite{ImageNetS}, ImageNet-A~\cite{ImageNetA} and ImageNet-R~\cite{ImageNetR} as the target domains.

\noindent\textbf{Compared Methods.} We propose to use an OOD detector to enhance the generalization ability of the existing VLPT methods. To validate the versatility of our method, we use four existing methods as the baselines: CoOp~\cite{CoOp}, CoCoOp~\cite{zhou2022cocoop}, ProGrad~\cite{ProGrad} and KgCoOp~\cite{kgcoop23}. The results of these methods with and without OOD detector will be compared. Besides, zero-shot CLIP is also compared.

\noindent\textbf{Implementation Details.} With the public code of CoOp, CoCoOp, ProGrad and KgCoOp, we reproduce the results on the aforementioned datasets using the suggested parameters. For CoOp, the context length is 16 and random initialization is adopted. The batch size is 32 and 50 epochs are trained. For CoCoOp, the context is initialized by ``a photo of a", batch size is 1 and 10 epochs are trained. For ProGrad, the training setting is identical to CoOp and the two extra hyper-parameters are set to 1.0. For KgCoOp, the context is initialized by ``a photo of a" and the additional hyper-parameter $\lambda$ is set to be 8.0. Due to limited GPU memory, the batch size is set to be 32. Besides, 100 epochs are trained. All methods use the same learning rate scheduler and optimizer. Our reproduced results are slightly different from previously reported, however, the differences are minor. For our method, we insert the proposed test-time technique into these methods and conduct model inference on test sets. The entropy based ID score and $\alpha=64$ are used as default. We report the final performance averaged over three random seeds. The same evaluation metrics adopted in CLIP are used. All experiments are conducted on RTX A4000 GPU. Without further specification, ViT-B/16 is adopted as the default vision backbone.

\subsection{Base to Novel Generalization}
In this setting, we split the train and test samples in 11 datasets into two groups: base classes (\texttt{Base}) and novel classes (\texttt{Novel}). The two sets do not share any identical classes. We train VLPT methods on base classes and evaluate them on novel classes. When a VLPT method is enhanced by OOD detection, the corresponding method is annotated by a plus sign (+) in the name. The detailed results of 9 methods are reported in Table~\ref{tab:base_to_novel}.

Table~\ref{tab:base_to_novel}(a) shows the average results over 11 datasets, where zero-shot CLIP acquires the worst accuracy on base classes. This is because it uses fixed prompts which are sub-optimal compared to learned prompts on base set. However, CLIP's accuray on novel classes is quite appealing (better than CoOp, CoCoOp, ProGrad and KgCoOp), demonstrating its strong generalization ability across different classes.

CoOp obtains the highest \texttt{Base} accuracy, implying its strong fitting capability on base classes. In contrast, its \texttt{Novel} accuracy is the worst, showing its low base-to-novel generalization ability. By introducing OOD detection into CoOp, the \texttt{Novel} accuracy is improved significantly with sacrificing the \texttt{Base} accuracy a little. As such, the harmonic mean increases a lot, as large as 2.6\%. As for CoCoOp, CoCoOp+ improves the harmonic mean by 0.7\% with a slight drop of \texttt{Base} accuracy (i.e. 0.86\%) and a remarkable increase of \texttt{Novel} accuracy (i.e. 2.01\%). Besides, ProGrad+ improves the harmonic mean of ProGrad by 1.5\% and the \texttt{Novel} accuracy by 3.85\% with a slight sacrifice of \texttt{Base} accuracy, i.e. 1.53\%. It obtains a new state of the art of harmonic mean, i.e. 77.20\%, outperforming KgCoOp's best result, 76.97\%. Finally, KgCoOp+ and KgCoOp acquire the similar harmonic mean. One of the reasons the proposed method has little effect on KgCoOp is that it adds an extra loss function to enforce the few-shot and zero-shot classifier approach each other, which is less friendly for model fusion. Another reason is KgCoOp alreadly achieves excellent harmonic mean, i.e. 76.97\%, which is very near to the upper bound harmonic mean of fusing KgCoOp and zero-shot CLIP, i.e. 77.67\%. This small gap makes further improvement harder. By summarizing the above observations, OOD detection is useful in improving base-to-novel generalization ability of existing VLPT methods.

Table~\ref{tab:base_to_novel}(b)-(l) show results on each dataset. First, CoOp+, CoCoOp+, ProGrad+ and KgCoOp+ reduce the \texttt{Base} performance compared to CoOp, CoCoOp, ProGrad and KgCoOp on 6/11 datasets, 9/11 datasets, 10/11 datasets and 11/11 datasets, respectively. However, as shown in Table~\ref{tab:base_to_novel}(a), the average \texttt{Base} performance drops a little especially for CoOp and CoCoOp. Second, they improve the \texttt{Novel} performance on 11/11 datasets, 9/11 datasets, 11/11 datasets and 7/11 datasets, respectively, implying boosted performance on new classes. Finally, they improve the harmonic mean performance on 11/11 datasets, 8/11 datasets, 11/11 datasets and 4/11 datasets, respectively. Specially, CoOp+ obtains an obivous improvement of harmonic mean of 1.81\%, 2.25\%, 7.12\%, 2.28\%, 2.69\%, 2.07\% and 6.89\% on ImageNet, StandfordCars, Flowers102, FGVCAircraft, SUN397, EuroSAT and UCF101 upon CoOp, respectively. In summary, the results demonstrate OOD detection can improve base-to-novel generalization ability.

The proposed dynamic fusion is also suitable for two few-shot classifiers with simple modification. By replacing zero-shot CLIP with ProGrad in KgCoOp+, the OOD detection based dynamic fusion (Pro+Kg) can obtain better results compared to both KgCoOp and ProGrad under the base-to-novel generalization setting, as can be seen from Table~\ref{tab:base_to_novel}. Actually, Pro+Kg achieves 0.4\% higher harmonic accuracy compared to the best result reported above, i.e. 77.20\%. Such an effectiveness is due to the strong complementarity between KgCoOp and ProGrad.

\subsection{Domain Generalization}
Domain generalization evaluates the effect of distribution shift under the same classes. The prompts are learned on source domain (i.e., ImageNet) and transferred to target domains (i.e., ImageNetV2, etc.) for evaluation. The results of zero-shot CLIP and four prompt tuning methods with and without OOD detection are reported in Table~\ref{tab:dg}. From the results, OOD detection may reduce the performances on target domains slightly in some cases. Fortunately, the drops are marginal compared to the improvements in Table~\ref{tab:base_to_novel}. In Table~\ref{tab:dg}, almost all results of CoOp/CoCoOp/ProGrad/KgCoOp surpass CLIP. This is reasonable as the source and target domains share the classes, learning on source domain has positive effect on improving the recognition accuracy on target domains. As such, when fusing the zero-shot and few-shot classifier, more weights of the latter will gather around one. Meantime, the weaker zero-shot classifier may bring side effect for fusion. These can explain why OOD detection does not demonstrate significant improvements in this setting. Actually, OOD detection also works well in this setting if there is strong complementarity between the two classifiers being fused. From Table~\ref{tab:dg}, KgCoOp and ProGrad have strong complementarity on different datasets. As such, we replace zero-shot CLIP with ProGrad in KgCoOp+. This method is denoted as Pro+Kg. From Table~\ref{tab:dg}, Pro+Kg can achieve a new state of art, i.e. 60.36\% average accuracy on target domains. Meantime, the accuracy on source domain is also appealing. The above results suggest that OOD detection is also helpful for domain generalization. 

\begin{table}[t]
	\tabstyle{10pt}
	\renewcommand{\arraystretch}{0.8}
	\begin{tabular}{l c c c | c}
		\toprule
		Method & Weight & Base & Novel & H \\
		\midrule
		\multirow{ 5}{*}{Static} & 0.95 & 82.47 & 68.48 & 74.83 \\
		&0.75 & 81.07 & 70.76 &  75.56 \\
		&0.50 & 78.70 & 73.13 & 75.81 \\
		&0.25 & 75.40 & 74.85 & 75.12 \\
		&0.05 & 70.51 & 74.37 & 72.39 \\
		\midrule
		Dynamic&$\alpha\!=\!64$ & 82.20 & 72.15 & \textbf{76.85} \\
		\bottomrule
	\end{tabular}
\caption{Dynamic weighting vs static weighting of CoOp+.}
\label{tab:dynamic_vs_static}
\end{table}

\begin{table}[t]
	\tabstyle{4pt}
		\renewcommand{\arraystretch}{0.8}
		\begin{tabular}{c | l | c c c}
			\toprule
			No. & OOD Detector & Base & Novel & H \\
			\midrule
			0 & MSP of zero-shot classifier only & 78.74 & 73.61 & 76.09 \\
			1 & MSP of few-shot classifier only & 81.06 & 72.30 & 76.43 \\
			\midrule
			2 & MSP (best $\alpha=64$) & 82.11 & 72.76 & \textbf{77.15}  \\
			3 & MSP ($\alpha=\infty$) & 82.13 & 72.70 &  77.13 \\
			\midrule
			4 & MaxLogit (best $\alpha=8$) & 78.45 & 73.52 &75.86 \\
			5 & MaxLogit ($\alpha=\infty$) & 76.64 & 72.53 & 74.53 \\
			\midrule
			6 & Energy (best $\alpha=64$) & 78.70 & 73.19 &  75.85 \\
			7 & Energy ($\alpha=\infty$) & 75.66 & 72.37 & 73.98 \\
			\midrule
			8 & Entropy (best $\alpha=64$) & 82.20 & 72.15 & 76.85 \\
			9 & Entropy ($\alpha=\infty$) & 82.22 & 72.03 & 76.79  \\
			\bottomrule
		\end{tabular}
\caption{Different OOD detection methods in CoOp+.}
\label{tab:ood_scoring_func}
\end{table}

\begin{figure}[h!]
	\captionsetup[subfigure]{aboveskip=-1pt,belowskip=-1pt}
	\centering
	\begin{subfigure}[t]{0.49\linewidth}
		\includegraphics[width=1.0\linewidth]{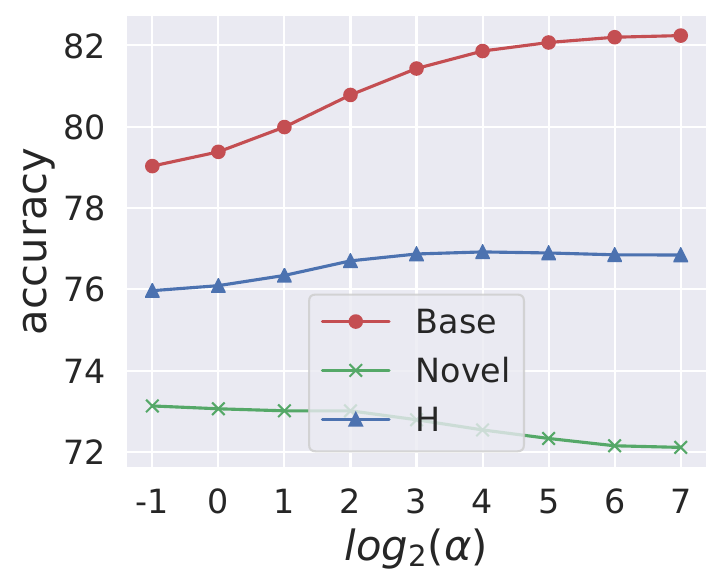}
		\caption{CoOp+ with Entropy.}
		\label{fig:CoOp_alpha_base_new}
	\end{subfigure}
	\hspace{-2pt}
	\begin{subfigure}[t]{0.49\linewidth}
		\includegraphics[width=1.0\linewidth]{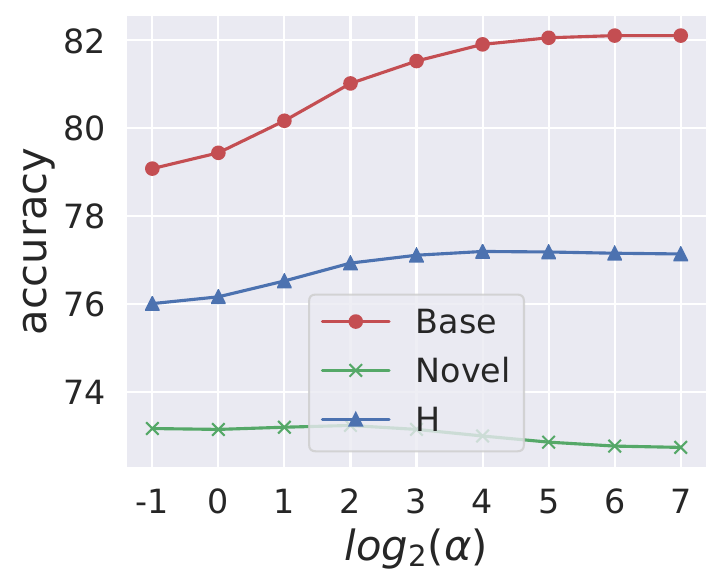}
		\caption{CoOp+ with  MSP.}
		\label{fig:CoOp_MSP_base_new}
	\end{subfigure}
	\caption{Effect of $\alpha$ on accuracies in CoOp+.}
	\label{fig:alpha_tradeoff}
\end{figure}

\begin{table}[t]
	\tabstyle{3.2pt}
		\renewcommand{\arraystretch}{0.9}
		\begin{tabular}{l c c c c c c}
			\toprule
			\multirow{2}{*}{Method} & \multirow{2}{*}{Backbone} & \multicolumn{3}{c}{Base-to-Novel} & \multicolumn{2}{c}{Domain}\\
			\cmidrule(lr){3-5} \cmidrule(lr){6-7}
			&  & Base & Novel & H & Source & Target\\
			\midrule
			CoOp& ResNet-50&77.17 & 59.03 & 66.89&62.99&41.05\\
			CoOp+& ResNet-50& \textbf{77.25} & \textbf{64.02} & \textbf{70.20} &\textbf{63.08} &\textbf{42.13}\\
			&&&{+4.99}&{+3.31}&&{+1.08}\\
			\midrule
			CoCoOp& ResNet-50 & \textbf{74.68} & 63.48 & 68.63 &\textbf{63.16}&\textbf{42.95}\\
			CoCoOp+&ResNet-50&74.28 & \textbf{66.49}& \textbf{70.17} & 62.34 & 42.90 \\
			&&&{+3.01}&{+1.54}&&{-0.05}\\
			\midrule
			ProGrad & ResNet-50 & \textbf{77.14} & 63.37 & 69.58 & \textbf{63.66} & 42.11 \\
			ProGrad+ & ResNet-50 & 76.09 & \textbf{67.39} & \textbf{71.48} & 62.92 & \textbf{42.67} \\
			&&&{+4.02}&{+1.90}&&{+0.56}\\
			\midrule
			KgCoOp & ResNet-50 & \textbf{76.42}& 64.93& 70.21& 62.52& 43.18\\
			KgCoOp+ & ResNet-50 & 74.52 & \textbf{68.01}& \textbf{71.12}& 61.33& 42.45\\
			&&&{+3.08}&{+0.91}&&{-0.73}\\
			\midrule
			\midrule
			CoOp & ViT-B/32& \textbf{79.24} & 62.77 & 70.05 &\textbf{66.72} &48.32\\
			CoOp+ & ViT-B/32& 78.99 & \textbf{67.85} & \textbf{73.00} &66.68&\textbf{49.21}\\
			& && {+5.08}&{+2.95}&&{+0.89}\\
			\midrule
			CoCoOp & ViT-B/32 & \textbf{76.26}& 66.63& 71.12& \textbf{66.07} & \textbf{49.89}\\
			CoCoOp+ & ViT-B/32 & 75.55& \textbf{70.58}& \textbf{72.98}& 65.39& 49.84\\
			& & &{+3.95}&{+1.86}&&{-0.05}\\
			\midrule
			ProGrad & ViT-B/32 & \textbf{78.75}& 66.19& 71.93& \textbf{66.98}& 49.33\\
			ProGrad+ & ViT-B/32 & 77.24 & \textbf{70.48}&\textbf{ 73.71}& 66.21& \textbf{49.59}\\
			& & & {+4.29}&{+1.78}&&{+0.26}\\
			\midrule
			KgCoOp & ViT-B/32 & \textbf{78.29} & 69.87& 73.84& \textbf{65.57}&\textbf{49.79} \\
			KgCoOp+ & ViT-B/32 & 76.37 & \textbf{71.61}& \textbf{73.91}& 64.68& 49.27\\
			& &&{+1.74}&{+0.07}&&{-0.52}\\
			\bottomrule
		\end{tabular}
	\caption{Base-to-novel and domain generalization results with different vision backbones for 16 shots. Mean accuracies over four target domains are reported. }
	\label{tab:resnet50_coop}
\end{table}

\subsection{Ablation Study}
\textbf{Static Weighting:} We can also use static (fixed) weight $s(\mathbf{x})$ for all test samples. The results of using static weights ($s(\mathbf{x})=0.05,0.25,0.50,0.75,0.95$) and dynamic weights ($\alpha=64$) are compared in Table~\ref{tab:dynamic_vs_static}. Using fixed weights also implies the OOD detectors with 50\% accuracy are adopted. The improved result of dynamic weighting demonstrates that weak OOD detectors help to boost generalizability across different classes.

\noindent\textbf{Different OOD Detection Methods:} The proposed method is flexible, it supports many different OOD detection methods. Besides the aforementioned ID scores, the maximum softmax probability (MSP) of the zero-shot and few-shot classifier can also serve as the weight, i.e., $s(\mathbf{x})=1-ids^{zs}(\mathbf{x})$ and $s(\mathbf{x})=ids^{fs}(\mathbf{x})$. The results of different methods are compared in Table~\ref{tab:ood_scoring_func}. When using sigmoid function, the results with the best $\alpha$ in $\{0.5,1,2,4,8,16,32,64,128\}$ and infinite $\alpha$ are reported. From the table, MSP (Row 2) achieves the best result, slightly better than entropy based method (Row 8). For different ID scores, using sigmoid function with smaller $\alpha$ is better than Heaviside step function ($\alpha=\infty$). The underlying reason is that continuous weights are helpful for model fusion compared to binary weights. We also find that MaxLogit and Energy do not work well, which is contrary to the observations in prior work of OOD detection. We conjecture that, for these methods, the ID scores from different classifiers are not well normalized for comparison.

\noindent\textbf{Effect of Scaling Factor $\boldsymbol{\alpha}$:} The $\alpha$ is the scaling factor of the score difference. Its effect for entropy and MSP based methods are shown in Fig.~\ref{fig:alpha_tradeoff}. Larger $\alpha$ damages the performance on novel classes, but improves the performance on base classes. Thus, it in fact plays a role in balancing the two performances. This makes the balance tunable compared to using Heaviside step function. 

\noindent\textbf{Different Backbones:} The proposed method also works for other vision backbones. To validate this point, the base-to-novel and domain generalization results with ResNet-50 and ViT-B/32 are presented in Table~\ref{tab:resnet50_coop}. With OOD detection, the harmonic mean accuracies in base-to-novel setting are improved uniformly (0.07\%$\sim$3.31\%). Meantime, methods with OOD detection can achieve similar or better mean accuracies on target domains in most cases (6 out of 8).

\section{Conclusion}
Existing methods for vision-language prompt tuning have poor generalization ability. This work proposes a new viewpoint for balancing the base and novel performance across classes by considering this problem as an OOD detection problem. Based on this idea, we propose a competition based scoring function to generate dynamic weight of zero-shot and few-shot classifier. By extensive experiments, even with weak OOD detectors, fused classifier demonstrates better generalization ability compared to the original one. One limitation of our method is that the effectiveness may be less significant when the two classifiers being fused have little complementarity. Future work will focus on exploiting stronger OOD detectors and addressing the above limitation.

\bibliography{aaai24}

\begin{thebibliography}{40}
\providecommand{\natexlab}[1]{#1}

\bibitem[{Bossard, Guillaumin, and Gool(2014)}]{Food101}
Bossard, L.; Guillaumin, M.; and Gool, L.~V. 2014.
\newblock Food-101 - Mining Discriminative Components with Random Forests.
\newblock In \emph{ECCV}, 446--461.

\bibitem[{Chen et~al.(2023)Chen, Yao, Song, Li, Rao, and Zhang}]{PLOT}
Chen, G.; Yao, W.; Song, X.; Li, X.; Rao, Y.; and Zhang, K. 2023.
\newblock {PLOT:} Prompt Learning with Optimal Transport for Vision-Language
  Models.
\newblock In \emph{ICLR}.

\bibitem[{Chen et~al.(2020)Chen, Dai, Liu, Chen, Yuan, and Liu}]{ChenDLCYL20}
Chen, Y.; Dai, X.; Liu, M.; Chen, D.; Yuan, L.; and Liu, Z. 2020.
\newblock Dynamic Convolution: Attention Over Convolution Kernels.
\newblock In \emph{CVPR}, 11027--11036.

\bibitem[{Cimpoi et~al.(2014)Cimpoi, Maji, Kokkinos, Mohamed, and
  Vedaldi}]{DTD}
Cimpoi, M.; Maji, S.; Kokkinos, I.; Mohamed, S.; and Vedaldi, A. 2014.
\newblock Describing Textures in the Wild.
\newblock In \emph{CVPR}, 3606--3613.

\bibitem[{Deng et~al.(2009)Deng, Dong, Socher, Li, Li, and
  Fei{-}Fei}]{ImageNet}
Deng, J.; Dong, W.; Socher, R.; Li, L.; Li, K.; and Fei{-}Fei, L. 2009.
\newblock ImageNet: {A} large-scale hierarchical image database.
\newblock In \emph{CVPR}, 248--255.

\bibitem[{Ding et~al.(2022)Ding, Wang, Liu, Yu, Zhang, Xiang, and
  Pan}]{SoftCPT}
Ding, K.; Wang, Y.; Liu, P.; Yu, Q.; Zhang, H.; Xiang, S.; and Pan, C. 2022.
\newblock Prompt Tuning with Soft Context Sharing for Vision-Language Models.
\newblock \emph{CoRR}, abs/2208.13474.

\bibitem[{Gao et~al.(2022)Gao, Zhang, Yang, Li, Gao, and Wen}]{ImageNetA}
Gao, H.; Zhang, H.; Yang, X.; Li, W.; Gao, F.; and Wen, Q. 2022.
\newblock Generating natural adversarial examples with universal perturbations
  for text classification.
\newblock \emph{Neurocomputing}, 471: 175--182.

\bibitem[{Gao et~al.(2023)Gao, Chen, Zhang, Xiao, and Sun}]{CompPT}
Gao, K.; Chen, L.; Zhang, H.; Xiao, J.; and Sun, Q. 2023.
\newblock Compositional Prompt Tuning with Motion Cues for Open-vocabulary
  Video Relation Detection.
\newblock In \emph{ICLR}.

\bibitem[{Helber et~al.(2019)Helber, Bischke, Dengel, and Borth}]{EuroSAT}
Helber, P.; Bischke, B.; Dengel, A.; and Borth, D. 2019.
\newblock {EuroSAT}: {A} Novel Dataset and Deep Learning Benchmark for Land Use
  and Land Cover Classification.
\newblock \emph{{IEEE} J. Sel. Top. Appl. Earth Obs. Remote. Sens.}, 12(7):
  2217--2226.

\bibitem[{Hendrycks et~al.(2022)Hendrycks, Basart, Mazeika, Zou, Kwon,
  Mostajabi, Steinhardt, and Song}]{MaxLogit}
Hendrycks, D.; Basart, S.; Mazeika, M.; Zou, A.; Kwon, J.; Mostajabi, M.;
  Steinhardt, J.; and Song, D. 2022.
\newblock Scaling Out-of-Distribution Detection for Real-World Settings.
\newblock In \emph{ICML}, 8759--8773.

\bibitem[{Hendrycks et~al.(2021)Hendrycks, Basart, Mu, Kadavath, Wang, Dorundo,
  Desai, Zhu, Parajuli, Guo, Song, Steinhardt, and Gilmer}]{ImageNetR}
Hendrycks, D.; Basart, S.; Mu, N.; Kadavath, S.; Wang, F.; Dorundo, E.; Desai,
  R.; Zhu, T.; Parajuli, S.; Guo, M.; Song, D.; Steinhardt, J.; and Gilmer, J.
  2021.
\newblock The Many Faces of Robustness: {A} Critical Analysis of
  Out-of-Distribution Generalization.
\newblock In \emph{ICCV}, 8320--8329.

\bibitem[{Hendrycks and Gimpel(2017)}]{MSP}
Hendrycks, D.; and Gimpel, K. 2017.
\newblock A Baseline for Detecting Misclassified and Out-of-Distribution
  Examples in Neural Networks.
\newblock In \emph{ICLR}.

\bibitem[{Hou and Kung(2021)}]{HouK21}
Hou, Z.; and Kung, S. 2021.
\newblock Parameter Efficient Dynamic Convolution via Tensor Decomposition.
\newblock In \emph{BMVC}, 107.

\bibitem[{Jia et~al.(2021)Jia, Yang, Xia, Chen, Parekh, Pham, Le, Sung, Li, and
  Duerig}]{ALIGN}
Jia, C.; Yang, Y.; Xia, Y.; Chen, Y.; Parekh, Z.; Pham, H.; Le, Q.~V.; Sung,
  Y.; Li, Z.; and Duerig, T. 2021.
\newblock Scaling Up Visual and Vision-Language Representation Learning With
  Noisy Text Supervision.
\newblock In \emph{ICML}, volume 139, 4904--4916.

\bibitem[{Krause et~al.(2013)Krause, Stark, Deng, and Fei{-}Fei}]{StanfordCars}
Krause, J.; Stark, M.; Deng, J.; and Fei{-}Fei, L. 2013.
\newblock {3D} Object Representations for Fine-Grained Categorization.
\newblock In \emph{ICCVW}, 554--561.

\bibitem[{Li, Zhou, and Yao(2022)}]{li2022omnidimensional}
Li, C.; Zhou, A.; and Yao, A. 2022.
\newblock Omni-Dimensional Dynamic Convolution.
\newblock In \emph{ICLR}.

\bibitem[{Li, Fergus, and Perona(2007)}]{Caltech101}
Li, F.; Fergus, R.; and Perona, P. 2007.
\newblock Learning generative visual models from few training examples: {A}n
  incremental Bayesian approach tested on 101 object categories.
\newblock \emph{Comput. Vis. Image Underst.}, 106(1): 59--70.

\bibitem[{Li et~al.(2022)Li, Li, Xiong, and Hoi}]{BLIP}
Li, J.; Li, D.; Xiong, C.; and Hoi, S. 2022.
\newblock BLIP: Bootstrapping Language-Image Pre-training for Unified
  Vision-Language Understanding and Generation.
\newblock In \emph{ICML}.

\bibitem[{Liu et~al.(2020)Liu, Wang, Owens, and Li}]{EnergyBasedOOD}
Liu, W.; Wang, X.; Owens, J.~D.; and Li, Y. 2020.
\newblock Energy-based Out-of-distribution Detection.
\newblock In \emph{NeurIPS}.

\bibitem[{Liu et~al.(2022)Liu, Wang, Han, Ma, and Gao}]{DyQuan}
Liu, Z.; Wang, Y.; Han, K.; Ma, S.; and Gao, W. 2022.
\newblock Instance-Aware Dynamic Neural Network Quantization.
\newblock In \emph{CVPR}, 12424--12433.

\bibitem[{Maji et~al.(2013)Maji, Rahtu, Kannala, Blaschko, and
  Vedaldi}]{FGVCAircraft}
Maji, S.; Rahtu, E.; Kannala, J.; Blaschko, M.~B.; and Vedaldi, A. 2013.
\newblock Fine-Grained Visual Classification of Aircraft.
\newblock \emph{CoRR}, abs/1306.5151.

\bibitem[{Mokady, Hertz, and Bermano(2021)}]{ClipCap}
Mokady, R.; Hertz, A.; and Bermano, A.~H. 2021.
\newblock ClipCap: {CLIP} Prefix for Image Captioning.
\newblock \emph{CoRR}, abs/2111.09734.

\bibitem[{Nilsback and Zisserman(2008)}]{Flowers102}
Nilsback, M.; and Zisserman, A. 2008.
\newblock Automated Flower Classification over a Large Number of Classes.
\newblock In \emph{ICVGIP}, 722--729.

\bibitem[{Parkhi et~al.(2012)Parkhi, Vedaldi, Zisserman, and
  Jawahar}]{OxfordPets}
Parkhi, O.~M.; Vedaldi, A.; Zisserman, A.; and Jawahar, C.~V. 2012.
\newblock Cats and dogs.
\newblock In \emph{CVPR}, 3498--3505.

\bibitem[{Radford et~al.(2021)Radford, Kim, Hallacy, Ramesh, Goh, Agarwal,
  Sastry, Askell, Mishkin, Clark, Krueger, and Sutskever}]{CLIP}
Radford, A.; Kim, J.~W.; Hallacy, C.; Ramesh, A.; Goh, G.; Agarwal, S.; Sastry,
  G.; Askell, A.; Mishkin, P.; Clark, J.; Krueger, G.; and Sutskever, I. 2021.
\newblock Learning Transferable Visual Models From Natural Language
  Supervision.
\newblock In \emph{ICML}, volume 139, 8748--8763.

\bibitem[{Ramesh et~al.(2022)Ramesh, Dhariwal, Nichol, Chu, and Chen}]{DALLE2}
Ramesh, A.; Dhariwal, P.; Nichol, A.; Chu, C.; and Chen, M. 2022.
\newblock Hierarchical Text-Conditional Image Generation with {CLIP} Latents.
\newblock \emph{CoRR}, abs/2204.06125.

\bibitem[{Rao et~al.(2022)Rao, Zhao, Chen, Tang, Zhu, Huang, Zhou, and
  Lu}]{DenseCLIP}
Rao, Y.; Zhao, W.; Chen, G.; Tang, Y.; Zhu, Z.; Huang, G.; Zhou, J.; and Lu, J.
  2022.
\newblock {DenseCLIP}: Language-Guided Dense Prediction with Context-Aware
  Prompting.
\newblock In \emph{CVPR}, 18082--18091.

\bibitem[{Recht et~al.(2019)Recht, Roelofs, Schmidt, and Shankar}]{ImageNetV2}
Recht, B.; Roelofs, R.; Schmidt, L.; and Shankar, V. 2019.
\newblock Do ImageNet Classifiers Generalize to ImageNet?
\newblock In \emph{ICML}, volume~97, 5389--5400.

\bibitem[{Shazeer et~al.(2017)Shazeer, Mirhoseini, Maziarz, Davis, Le, Hinton,
  and Dean}]{ShazeerMMDLHD17}
Shazeer, N.; Mirhoseini, A.; Maziarz, K.; Davis, A.; Le, Q.~V.; Hinton, G.~E.;
  and Dean, J. 2017.
\newblock Outrageously Large Neural Networks: The Sparsely-Gated
  Mixture-of-Experts Layer.
\newblock In \emph{ICLR}.

\bibitem[{Shu et~al.(2022)Shu, Nie, Huang, Yu, Goldstein, Anandkumar, and
  Xiao}]{TTPT}
Shu, M.; Nie, W.; Huang, D.; Yu, Z.; Goldstein, T.; Anandkumar, A.; and Xiao,
  C. 2022.
\newblock Test-Time Prompt Tuning for Zero-Shot Generalization in
  Vision-Language Models.
\newblock In \emph{NeurIPS}.

\bibitem[{Soomro, Zamir, and Shah(2012)}]{UCF101}
Soomro, K.; Zamir, A.~R.; and Shah, M. 2012.
\newblock {UCF101:} {A} Dataset of 101 Human Actions Classes From Videos in The
  Wild.
\newblock \emph{CoRR}, abs/1212.0402.

\bibitem[{Sun, Hu, and Saenko(2022)}]{DualCoOp}
Sun, X.; Hu, P.; and Saenko, K. 2022.
\newblock {DualCoOp}: Fast Adaptation to Multi-Label Recognition with Limited
  Annotations.
\newblock \emph{CoRR}, abs/2206.09541.

\bibitem[{Wang et~al.(2019)Wang, Ge, Lipton, and Xing}]{ImageNetS}
Wang, H.; Ge, S.; Lipton, Z.~C.; and Xing, E.~P. 2019.
\newblock Learning Robust Global Representations by Penalizing Local Predictive
  Power.
\newblock In \emph{NeurIPS}, 10506--10518.

\bibitem[{Xiao et~al.(2010)Xiao, Hays, Ehinger, Oliva, and Torralba}]{SUN397}
Xiao, J.; Hays, J.; Ehinger, K.~A.; Oliva, A.; and Torralba, A. 2010.
\newblock {SUN} database: Large-scale scene recognition from abbey to zoo.
\newblock In \emph{CVPR}, 3485--3492.

\bibitem[{Yang et~al.(2021)Yang, Zhou, Li, and Liu}]{OODsurvey}
Yang, J.; Zhou, K.; Li, Y.; and Liu, Z. 2021.
\newblock Generalized Out-of-Distribution Detection: {A} Survey.
\newblock \emph{CoRR}, abs/2110.11334.

\bibitem[{Yao, Zhang, and Xu(2023)}]{kgcoop23}
Yao, H.; Zhang, R.; and Xu, C. 2023.
\newblock Visual-Language Prompt Tuning with Knowledge-guided Context
  Optimization.
\newblock In \emph{CVPR}.

\bibitem[{Zhang and Xiang(2023)}]{DecoupleMaxLogit}
Zhang, Z.; and Xiang, X. 2023.
\newblock Decoupling MaxLogit for Out-of-Distribution Detection.
\newblock In \emph{CVPR}, 3388--3397.

\bibitem[{Zhou et~al.(2022{\natexlab{a}})Zhou, Yang, Loy, and
  Liu}]{zhou2022cocoop}
Zhou, K.; Yang, J.; Loy, C.~C.; and Liu, Z. 2022{\natexlab{a}}.
\newblock Conditional Prompt Learning for Vision-Language Models.
\newblock In \emph{CVPR}, 16816--16825.

\bibitem[{Zhou et~al.(2022{\natexlab{b}})Zhou, Yang, Loy, and Liu}]{CoOp}
Zhou, K.; Yang, J.; Loy, C.~C.; and Liu, Z. 2022{\natexlab{b}}.
\newblock Learning to Prompt for Vision-Language Models.
\newblock \emph{IJCV}, 130(9): 2337--2348.

\bibitem[{Zhu et~al.(2023)Zhu, Niu, Han, Wu, and Zhang}]{ProGrad}
Zhu, B.; Niu, Y.; Han, Y.; Wu, Y.; and Zhang, H. 2023.
\newblock Prompt-aligned Gradient for Prompt Tuning.
\newblock ICCV.

\end{thebibliography}

\end{document}